\documentclass[11pt,reqno]{amsart}
\usepackage[margin=1in]{geometry}

\usepackage{natbib}

\usepackage{mathtools,amsthm,amssymb,mathdots}
\usepackage{enumerate}

\theoremstyle{definition}

\newtheorem{theorem}{Theorem}[section]
\newtheorem{lemma}[theorem]{Lemma}
\newtheorem{corollary}[theorem]{Corollary}

\usepackage{graphicx}

\newcommand{\mathbbm}[1]{\text{\usefont{U}{bbold}{m}{n}#1}}

\newcommand{\tp}{{\scriptscriptstyle\mathsf{T}}}

\usepackage{hyperref,xcolor}
\hypersetup{
    colorlinks,
    linkcolor={red!50!black},
    citecolor={blue!50!black},
    urlcolor={blue!80!black}
}

\begin{document}
\title[Neural LU and Toeplitz decompositions]{LU decomposition and Toeplitz decomposition\\
of a neural network}
\author[Y.~Liu]{Yucong Liu}
\address{Department of Statistics, University of Chicago, Chicago, IL 60637}
\email{yucongliu@uchicago.edu}
\author[S.~Jiao]{Simiao Jiao}
\address{Department of Statistics, University of Chicago, Chicago, IL 60637}
\email{smjiao@uchicago.edu}
\author[L.-H.~Lim]{Lek-Heng Lim}
\address{Computational and Applied Mathematics Initiative, University of Chicago, Chicago, IL 60637}
\email{lekheng@uchicago.edu}

\begin{abstract}
It is well-known that any matrix $A$ has an LU decomposition. Less well-known is the fact that it has a `Toeplitz decomposition' $A = T_1 T_2 \cdots T_r$ where $T_i$'s are Toeplitz matrices. We will prove that any continuous function $f : \mathbb{R}^n \to \mathbb{R}^m$ has an approximation to arbitrary accuracy by a neural network that takes the form $L_1 \sigma_1 U_1 \sigma_2 L_2 \sigma_3 U_2 \cdots L_r \sigma_{2r-1} U_r$, i.e., where the weight matrices alternate between lower and upper triangular matrices, $\sigma_i(x) \coloneqq \sigma(x - b_i)$ for some bias vector $b_i$, and the activation $\sigma$ may be chosen to be essentially any uniformly continuous nonpolynomial function. The same result also holds with Toeplitz matrices, i.e., $f \approx T_1 \sigma_1 T_2 \sigma_2 \cdots \sigma_{r-1} T_r$ to arbitrary accuracy, and likewise for Hankel matrices. A consequence of our Toeplitz result is a  fixed-width universal approximation theorem for convolutional neural networks, which so far have only  arbitrary width versions. Since our results apply in particular to the case when $f$ is a general neural network, we may regard them as LU and Toeplitz decompositions of a neural network. The practical implication of our results is that one may vastly reduce the number of weight parameters in a neural network without sacrificing its power of universal approximation. We will present several experiments on real data sets to show that imposing such structures on the weight matrices sharply reduces the number of training parameters with almost no noticeable effect on test accuracy.
\end{abstract}
\maketitle

\section{Introduction}

Among the numerous results used to justify and explain the efficacy of feed-forward neural networks, possibly the best known are the universal approroximation theorems of various stripes. These theorems explain the expressive power of neural networks by showing that they can approximate various classes of functions to arbitrary accuracy under various measures of accuracy. The universal approximation theorems in the literature may be  divided into two categories, applying respectively to:
\begin{enumerate}[\upshape (i)]
\item \emph{shallow wide networks:} neural networks of fixed depth and arbitrary width;
\item \emph{deep narrow networks:} neural networks with fixed width and arbitrary depth.
\end{enumerate}
In the first category, we have the celebrated results of \citet{cybenko1989approximation,hornik1991approximation,pinkus1999approximation}, et al. We state the last of these for easy reference:
\begin{theorem}[\citealp{pinkus1999approximation}]\label{thm:uat1}
Let $\sigma \in C(\mathbb{R})$ and $\Omega \subseteq \mathbb{R}^n$ be compact. The set of $\sigma$-activated neural networks with one hidden layer neural networks and arbitrary width is dense in $C(\Omega,\mathbb{R}^m)$ with respect to the uniform norm if and only if $\sigma$ is not a polynomial.
\end{theorem}
In the second category, an example is provided by \citet{kidger2020universal}, again quoted below for easy reference:
\begin{theorem}[\citealp{kidger2020universal}]\label{thm:uat2}
Let $\sigma \in C(\mathbb{R})$ be a nonpolynomial function, continuously differentiable with nonzero derivative on at least one point. Let $\Omega \subseteq \mathbb{R}^n$ be compact. Then the set of $\sigma$-activated neural networks with fixed width $m+n+1$ and arbitrary depth is dense in $C(\Omega, \mathbb{R}^m)$ with respect to the uniform norm.
\end{theorem}
In all these results, the weight matrices used in each layer are assumed to be dense general matrices; in particular, these neural networks are fully connected. The goal of our article is to show that even when we impose special structures on the weight matrices --- upper and lower triangular, Toeplitz or Hankel --- we will still have the same type of universal approximation results, for both shallow wide and deep narrow netowrks alike. In addition, our numerical experiments will show that when kept at the same depth and width, a neural network with these structured weight matrices suffers almost no loss in expressive powers, but requires only a fraction of the parameters --- note that an $m \times n$ triangular matrix with $p = \max(m,n)$ has at most $p(p+1)/2$ parameters whereas an $m \times n$ Toeplitz or Hankel matrix has exactly $m +n -1$ parameters.

The saving in training cost goes beyond a mere reduction in the number of weight parameters. The forward and backward propagations in the training process ultimately reduce to matrix-vector products. For Toeplitz or Hankel matrices, these come at a cost of $O(n \log n)$ operations as opposed to the usual $O(n^2)$.

An alternative way to view our results is that these are ``LU decomposition'' and ``Toeplitz decomposition'' of a nonlinear function in the context of neural networks. A departure from the case of a linear functions is that an LU decomposition of a nonlinear function requires not just one lower-triangular matrix and one upper-triangular matrix but several of these alternating between lower-triangular and upper-triangular, and sandwiching an activation. The Toeplitz (or Hankel) decomposition of a linear function is a consequence of the following result, which can be readily extended to $m \times n$ matrices, as we will see in Section~\ref{sec:deep}.
\begin{theorem}[\citealp{ye2016every}]\label{thm:toep}
Every $n \times n$ matrix can be expressed as a product of $2n + 5$ Toeplitz matrices or $2n + 5$ Hankel matrices.
\end{theorem}
Again we will see that this also applies to a nonlinear continuous function as long as we introduce an activation function between every Toeplitz or Hankel factor. Another caveat in these results is that the exact equality used in linear algebra is replaced by the most common notion of equality in approximation theory, namely, equality up to an arbitrarily small error.
As in Theorems~\ref{thm:uat1} and \ref{thm:uat2}, our results will apply with essentially any nonpolynomial continuous activations, including but not limited to common ones like ReLU, leaky ReLU, sigmoidal, hyperbolic tangent, etc.

We will prove these results in Section~\ref{sec:main}, with shallow wide neural networks in Section~\ref{sec:wide} and deep narrow neural networks Section~\ref{sec:deep}, after discussing prior works in Section~\ref{sec:related} and setting up notations in Section~\ref{sec:prelim}. The experiments showing the practical side of these results are in Section~\ref{sec:exp} with a cost analysis in Section~\ref{sec:cost}.

\subsection{Prior works}\label{sec:related}

We present a more careful discussions of existing works in the literature, in rough chronological order. To the best of our knowledge, there are six main lines of works related to ours. While none replicates our results in Section~\ref{sec:main}, they show a progression towards to our work in spirit --- with the increase in width and depth of neural networks, it has become an important endeavor to reduce the number of redundant training parameters through other means.

\subsubsection*{Shallow wide neural networks:} The earliest universal approximation theorems are for one-hidden-layer neural networks with arbitrary width, beginning with the eponymous theorem of \citet{cybenko1989approximation}, which shows that a fully-connected sigmoid-activated network with one hidden layer and arbitrary number of neurons can approximate any continuous function on the unit cube in $\mathbb{R}^n$ up to arbitrary accuracy. Cybenko's argument  also works for ReLU activation and could be extended to a fixed number of hidden layers simply by requiring that the additional hidden layers approximate an identity map. \citet{HORNIK1989359} obtained the next major generalization to nondecreasing activations with $\lim_{x \to -\infty} \sigma(x) = 0$ and $\lim_{x \to +\infty} \sigma(x) = 1$. The most general universal approximation theorem along this line is that of  \citet{pinkus1999approximation} stated earlier in Theorem~\ref{thm:uat1}. The striking aspect is that it is a necessary and sufficient condition, showing that such universal approximation property characterizes the ``nonpolynomialness'' of the activation function.

\subsubsection*{Deep narrow networks:} With the advent of deep neural networks, the focus has changed to keeping width fixed and allowing depth to increase. \citet{lu2017expressive} showed that  ReLU-activated neural networks of width $n+4$ and arbitrary depth are dense in $L^1(\mathbb{R}^n)$. \citet{hanin2017approximating} showed that such neural networks of width $m+n$ are dense in $C(\Omega, \mathbb{R}^m)$ for any compact $\Omega \subseteq \mathbb{R}^n$. The aforementioned Theorem~\ref{thm:uat2} of \citet{kidger2020universal} is another alternative with more general continuous activations and with width $m+n+1$. An extreme case is provided by \citet{lin2018resnet} for ResNet with a single neuron per hidden layer but with depth going to infinity.

\subsubsection*{Width-depth tradeoff:} The tradeoff between width and depth of a neural network is now well studied. The results of  \citet{eldan2016power, telgarsky2016benefits} explain the benefits of having more layers --- a deep neural network cannot be well approximated by shallow neural networks unless they are exponentially large.  On the other hand, the results of \citet{johnson2018deep, park2021minimum} revealed the limitations of deep neural networks --- they require a minimum width for  universal approximation; although these results do not cover exotic structures like ResNet. There are also studies on the memory capacity of wide and deep neural networks \citep{yun2019small, vershynin2020memory}.

\subsubsection*{Neural network pruning:} Pruning refers to techniques for eliminating redundant weights from neural networks and it has a long history \citep{lecun1989optimal, hassibi1992second, han2015learning, li2016pruning}. A recent highlight is  the lottery ticket hypothesis proposed in \citet{frankle2018the} that led to extensive follow-up work \citep{morcos2019one, frankle2020linear, malach2020proving}. Our results in Section~\ref{sec:main} may be viewed as a particularly aggressive type of pruning whereby we either set half the weight parameters to zero, as in the LU case, or even reduce the number of weight parameters by an order of magnitude, from $O(n^2)$ to $O(n)$, as in the Toeplitz/Hankel case.

\subsubsection*{Convolutional neural networks:} The result closest to ours is likely the universal approximation theorem for deep convolutional neural network of \citet{zhou2020universality}. However his result provides the necessary width and depth in terms of the approximating accuracy $\varepsilon$, and as such requires arbitrary width and depth at the same time. We will deduce an alternative version with fixed width in Corollary~\ref{cor:cnn}.

\subsubsection*{Hardware acceleration:} In the context of accelerating training of neural networks via GPUs, FPGAs, ASICs, and other specialized hardware (e.g., Google's TPU, Nvidia's H100 AI processor),  there have been prior works on exploiting structured matrix algorithms for matrix-vector multiply, notably for triangular matrices in \citep{inoue2019efficient} and Toeplitz matrices in \citep{kelefouras2014methodology}.

\subsection{Notations and conventions}\label{sec:prelim}

We write $\lVert\, \cdot\,\rVert $ for both the Euclidean norm on $\mathbb{R}^n$ and the Frobenius norm on $\mathbb{R}^{m \times n}$.  The zero matrix in $\mathbb{R}^{m\times n}$ is denoted $\mathbbm{0}_{m\times n}$. The zero vector and the vector of all ones in $\mathbb{R}^n$ will be denoted $\mathbbm{0}_n$ and $\mathbbm{1}_n $ respectively. 

Let $A = (a_{ij}) \in \mathbb{R}^{m\times n}$. If $a_{ij} = 0$ whenver $i>j$, then $A$ is \emph{upper triangular}; if $a_{ij} = 0$ whenever $i < j$, then $A$ is lower triangular. A matrix is  \emph{Toeplitz} (resp.\ \emph{Hankel}) if has equal entries along its diagonals (resp.\ reverse diagonals). More precisely, $A$ is Toeplitz if $a_{i,i+r} = a_{j,j+r}$ whenever $-m+1\leq r\leq n-1$, $1\leq i,j\leq m$, $1\leq i+r,j+r\leq n$. Similarly $A$ is Hankel if $a_{i,r-i} = a_{j,r-j}$ whenever $2\leq r\leq m+n$, $1\leq i,j\leq m$, $1\leq r-i,r-j,\le n$. Note that the definitions of these structured matrices do not require that $m =n$. 

An $m\times n$ Toeplitz or Hankel matrix requires only $m + n-1$ parameters to specify --- standard convention is to just store the first row and first column of a Toeplitz matrix and the first row and last column of a Hankel matrix. For example, when $m =n$, we have
\[
T=
\begin{bmatrix}
a_{0} & a_{-1} &  & a_{1-n}\\
a_{1} & a_{0} & \ddots & \\
& \ddots & \ddots & a_{-1}\\
a_{n-1} &  & a_{1} & a_{0}
\end{bmatrix},  \qquad
H=
\begin{bmatrix}
a_{0} & a_1 & \cdots & a_{n-1} \\ 
a_1 & a_2 & \iddots & a_{n} \\ 
\vdots & \iddots & \iddots &  \vdots \\ 
a_{n-1} & a_{n} &  \cdots & a_{2n-2}
\end{bmatrix}.
\]

We write $C(\Omega,\mathbb{R}^m)$ for the set of continuous functions on $\Omega$ taking values in $\mathbb{R}^m$, with $C(\Omega)$ for the special case when $m=1$. Throughout this article we will use the uniform norm for all function approximations; there will be no confusion with the norms introduced above as we will always specify our uniform norm explicitly as $\sup_{x\in \Omega}$.

Any univariate function $\sigma: \mathbb{R} \to \mathbb{R}$ defines a \emph{pointwise activation} $\sigma : \mathbb{R}^n \to \mathbb{R}^n$ for any $n \in \mathbb{N}$ through applying $\sigma$ coordinatewise to vectors in $\mathbb{R}^n$. We will sometimes drop the parentheses, writing $\sigma x$ to mean $\sigma(x)$, to reduce notational clutter.

A $k$-layer neural network $\nu:\mathbb{R}^n \rightarrow \mathbb{R}^m$ has the following structure:
\[
    \nu(x) = A_k \sigma_{k-1} A_{k-1} \sigma_{k-1} \cdots  \sigma_2 A_{2}\sigma_1 A_1 x + b_k
\]
for any input $x \in \mathbb{R}^n$, \emph{weight} matrix  $A_i \in \mathbb{R}^{n_i\times n_{i-1}}$, 
\[
\sigma_i(x) \coloneqq \sigma(x +b_i)
\]
with $b_i \in \mathbb{R}^{n_i}$ the \emph{bias} vector, and $\sigma$ the \emph{activation} function. 
The output size of the $i$th \emph{layer}  is  $n_i$ and always equals the input size of $(i+1)$th layer, with $n_0 = n$ and $n_k= m$. 

\section{Universal approximation by structured neural networks}\label{sec:main}

We present our main results and proofs, beginning with shallow wide networks and followed by deep narrow networks.

\subsection{Fixed depth, arbitrary width}\label{sec:wide}

This is an easy case that we state for completeness. Our universal approximation result in this case only holds for real-valued functions. The more interesting case for arbitrary depth neural networks in Section~\ref{sec:deep} will hold for vector-valued functions.

We begin with an observation that, if width is not a limitation, then any general weight matrix may be transformed into a Toeplitz or Hankel matrix.
\begin{lemma}[General matrices to Toeplitz/Hankel matrices]\label{lem:convert}
Any matrix $A \in \mathbb{R}^{m\times n}$ can be transformed into a Toeplitz or Hankel matrix by inserting additional rows. 
\end{lemma}
\begin{proof}
This is best illustrated by way of a simple example first. For a $2\times 2$ matrix
\[
    \begin{bmatrix}
  a_{11}& a_{12}\\
  a_{21}& a_{22}
\end{bmatrix},
\]
inserting a row vector in the middle makes it Toeplitz:
\[
    \begin{bmatrix}
  a_{11}& a_{12}\\
  a_{22}& a_{11}\\
  a_{21}& a_{22}
\end{bmatrix};
\]
and similarly inserting a different row vector in the middle makes it Hankel:
\[
    \begin{bmatrix}
  a_{11}& a_{12}\\
  a_{12}& a_{21}\\
  a_{21}& a_{22}
\end{bmatrix}.
\]
For an $m \times n$ matrix
\[
A = \begin{bmatrix}
  a_{11}& a_{12} & \cdots &a_{1n} \\
  a_{21}&a_{22}  & \cdots & a_{2n}\\
  \vdots& \vdots & \ddots & \vdots\\
  a_{m1}& a_{m2} &\cdots  & a_{mn}
\end{bmatrix}
\]
inserting $n-1$ rows between the first and the second row
\[
\begin{bmatrix}
  a_{11}& a_{12} & \cdots &a_{1n} \\
  a_{21}&a_{22}  & \cdots & a_{2n}\\
\end{bmatrix}
\]
turns it Toeplitz
\[
\begin{bmatrix}
  a_{11}& a_{12}  & \cdots  &a_{1n} \\
  a_{2n}&a_{11}  & \ddots  & \vdots \\
\vdots      &  \ddots& \ddots & a_{12} \\
  a_{22}&  &   a_{2n}&a_{11} \\
  a_{21}& a_{22}  &  \cdots &a_{2n}
\end{bmatrix}.
\]
Now repeat this to the remaining pairs of adjacent rows of $A$, we see that after inserting a total of $(m-1)(n-1)$ rows, we obtain a Toeplitz matrix.  The process for transforming a general $m \times n$ matrix into a Hankel matrix by inserting rows is similar.
\end{proof}
Evidently, the statement and proof of Lemma~\ref{lem:convert} remain true if `row' is replaced by `column' but we will only need the row version in our proofs.

\begin{theorem}[Universal approximation by structured neural networks I]
Let $\Omega \subseteq \mathbb{R}^n$ be compact and $\sigma:\mathbb{R}\to\mathbb{R}$ be nonploynomial. For any $f\in C(\mathbb{R}^n)$ and any $\varepsilon > 0$, we have
\[
\sup_{x\in \Omega} \;  |f(x) -\nu(x)| \le \varepsilon
\]
for some one-layer neural network $\nu : \mathbb{R}^n \to \mathbb{R}$,
\[
\nu(x) = a^\tp \sigma(Ax + b),
\]
with $a,b\in \mathbb{R}^m$, $m \in \mathbb{N}$, and $A \in \mathbb{R}^{m\times n}$ that can be chosen to be
\begin{enumerate}[\upshape (i)]
  \item a Toeplitz matrix,
  \item a Hankel matrix,
  \item or a lower triangular matrix.
\end{enumerate}
\end{theorem}

\begin{proof}
It follows from Theorem~\ref{thm:uat1} that for a given $f\in C(\mathbb{R}^n)$, there exist $c,d\in \mathbb{R}^{p}$, $p \in \mathbb{N}$, $B \in \mathbb{R}^{p\times n}$ so that 
\[
\sup_{x\in \Omega} \;  |f(x) - d^\tp \sigma(Bx + c)| \le \varepsilon.
\]
Here of course $B$ has no specific structure. We begin with the Toeplitz case. By Lemma~\ref{lem:convert}, we first transform $B \in \mathbb{R}^{p \times n}$ into  a Toeplitz matrix $A \in \mathbb{R}^{m \times n}$ for some $m \in \mathbb{N}$. Since $A$ is obtained from $B$ by inserting rows, let the rows $i_1,\dots, i_p$ of $A$ be rows $1,\dots,p$ of $B$. Now let $a \in \mathbb{R}^m$ be the vector whose $i_j$th entry is exactly the $j$th entry of $d$ and zeroes everywhere else. Likewise let $b \in \mathbb{R}^m$ be the vector whose $i_j$th entry is exactly the $j$th entry of $c$ and zeroes everywhere else. Then we clearly have $d^\tp \sigma(Bx + c) = a^\tp \sigma(Ax+b)$ and the required result follows. The Hankel case is identical.
For the remaining case, we set
\[
a =
\begin{bmatrix}
\mathbbm{0}_{n}\\
d
\end{bmatrix}, \quad
A =
\begin{bmatrix}
\mathbbm{0}_{n\times n} \\
B
\end{bmatrix}, \quad
b =
\begin{bmatrix}
\mathbbm{0}_n\\
c
\end{bmatrix},
\]
and observe that $d^\tp \sigma(Bx + c) = a^\tp \sigma(Ax+b)$. Hence the required result follows.
\end{proof}
Theorem~\ref{thm:uat1} is false if $A$ is required to be upper triangular.

\subsection{Fixed width, arbitrary depth}\label{sec:deep}

The one-layer arbitrary width case above is more of a curiosity. Modern neural networks are almost invariably multilayer and we now provide the result that applies to this case. We first show that the identity map on $\mathbb{R}^n$ may be approximated by essentially any continuous pointwise activation. This is a generalization of \citep[Lemma~4.1]{kidger2020universal}.
\begin{lemma}[Approximation of identity]\label{lem:approx}
Let $\sigma: \mathbb{R} \rightarrow \mathbb{R}$ be any continuous function that is continuously differentiable with nonzero derivative at some point $a \in \mathbb{R}^n$. Let $I: \mathbb{R}^n \to \mathbb{R}^n$ be the identity map. Then for any compact $\Omega \subseteq \mathbb{R}^n$ and any $\varepsilon > 0$, there exists a $\delta > 0$ such that whenever $0<|h|<\delta$, the function $\rho_h : \mathbb{R}^n \to \mathbb{R}^n$,
\begin{equation}\label{eq:rho}
    \rho_h(x) \coloneqq \frac{1}{h\sigma'(a)} [\sigma(hx + a\mathbbm{1}_n) - \sigma(a)\mathbbm{1}_n],
\end{equation}
satisfies
\[
\sup_{x\in \Omega} \;\lVert \rho_h(x) - I(x)\rVert  \le \varepsilon.
\]
\end{lemma}
\begin{proof}
Subscript $i$ in this proof refers to the $i$th coordinate. As $\Omega$ is compact,  $|x_i| \le L$ for some $L > 0$ and for all $i=1,\dots,n$. Since the derivative $\sigma'$ is continuous, there exists $\eta > 0$ such that
\[
|\sigma'(b) - \sigma'(a)| < \frac{\sigma'(a) \varepsilon}{L\sqrt{n}}
\]
whenever $|b-a| \le \eta$.
Let $\delta = \eta/L$. Then for $0<|h|<\delta$, we have
\[
    |\rho_h(x)_i - x_i| = \left| \frac{\sigma(a+hx_i) - \sigma(a)}{h\sigma'(a)} - x_i\right|
   = \left|\frac{x_i\sigma'(\xi)}{\sigma'(\alpha)}  - x_i\right| 
     \le L \left| \frac{\sigma'(\xi) - \sigma'(a)}{\sigma'(a)}\right| \le \frac{\varepsilon}{\sqrt{n}}
\]
for some $\xi$ between $a+hx_i$ and $a$ by the mean value theorem. The last inequality follows from $|\xi - a| \le |hx_1| \le |h|L \le \eta$. Note that $\delta$ is independent of all $x_i$'s and therefore $x$. Hence we may take $\sup_{x\in \Omega}\lVert \, \cdot \, \rVert$ to get the required result.
\end{proof}

The proof of our main result below depends on two things: that we may use $\rho_h$ to approximate the identity map; and that if we scale the input of our activation by $h$ or the output by $1/h\sigma'(a)$, it does not affect the structure of our weight matrices --- Toeplitz, Hankel, and triangular structures are preserved under scalar multiplication.
\begin{theorem}[Universal approximation by structured neural networks II]\label{thm:depth}
Let $\Omega \subseteq \mathbb{R}^n$ be compact and $\sigma: \mathbb{R} \rightarrow \mathbb{R}$ be any uniformly continuous nonpolynomial function continuously differentiable with nonzero derivative on at least one in $\Omega$. For any  $f \in C(\mathbb{R}^n, \mathbb{R}^m)$ and any $\varepsilon > 0$,
\[
\sup_{x\in \Omega} \; \lVert f(x)-\nu(x)\rVert < \varepsilon.
\]
for some  neural network $\nu : \mathbb{R}^n \to \mathbb{R}^m$,
\[
    \nu(x) = A_k \sigma_{k-1} A_{k-1} \sigma_{k-1} \cdots  \sigma_2 A_{2}\sigma_1 A_1 x + b_k,
\]
where the weight matrices $A_1 \in \mathbb{R}^{(m+n+1) \times n}$,
\[
A_2,\dots,A_{k-1} \in \mathbb{R}^{(m+n+1)\times (m+n+1)},
\]
and $A_k \in \mathbb{R}^{m \times (m+n+1)}$ may be chosen to be
\begin{enumerate}[\upshape (i)]
  \item all Toeplitz,
  \item all Hankel,
  \item upper triangular for odd $i$ and lower triangular for even $i$;
\end{enumerate}
with bias vectors $b_1,\dots,b_{k-1} \in \mathbb{R}^{m + n +1}$,  $b_k \in \mathbb{R}^m$, and $\sigma_i(x) \coloneqq \sigma( x+ b_i)$.
\end{theorem}
\begin{proof}
By Theorem~\ref{thm:uat2}, there is a neural network $\varphi$ of width $m+n+1$ such that $\sup_{x\in \Omega} \lVert f(x)-\varphi(x)\rVert < \varepsilon/2$. We will write $\varphi$ recursively as
\[
        \varphi(x)= B_k \varphi_k(x) + c_k
\]
with $\varphi_0(x) = x$ and $\varphi_{j+1}(x) = \sigma(B_j\varphi_j(x) + c_j)$, $j = 1,\dots, k-1$. Here $B_1 \in \mathbb{R}^{(m+n+1)\times n}$, $B_k \in \mathbb{R}^{m \times (m+n+1)}$, and $B_2, \dots,B_{k-1} \in \mathbb{R}^{(m+n+1)\times (m+n+1)}$.

By Theorem~\ref{thm:toep}, the square matrices $B_2, \dots,B_{k-1}$ may each be decomposed into a product of Toeplitz matrices:
\begin{equation}\label{eq:toepdecomp}
B_j = T^{(j)}_1T^{(j)}_2\cdots T^{(j)}_{r_j}.
\end{equation}
As for $B_1$, we have 
\[
B_1 = [B_1, \mathbbm{0}_{(m+n+1)\times (m+1)}] \begin{bmatrix} I_n \\ \mathbbm{0}_{(m+1) \times n} \end{bmatrix}
\]
and as $[I_n, \mathbbm{0}_{n\times (m+1)}]^\tp \in \mathbb{R}^{(m+n+1) \times n}$ is a rectangular Toeplitz matrix and Theorem~\ref{thm:toep} applies to the square matrix $[B_1, \mathbbm{0}_{(n+m+1)\times (m+1)}] \in \mathbb{R}^{(m+ n +1) \times (m + n +1)}$, we also have a Toeplitz decomposition for $B_1$. The argument applied to $B_1$ also applies to $B_k^\tp$. Hence we have
\[
B_1 = T^{(1)}_1\cdots T^{(1)}_{r_1}, \quad B_k = T^{(k)}_1\cdots T^{(k)}_{r_k}
\]
as well. We thus obtain
\[
\varphi(x) = T^{(k)}_1\cdots T^{(k)}_{r_k} \varphi_k(x) + c_k
\]
with $\varphi_0(x) = x$ and
\begin{equation}\label{eq:j}
    \varphi_j(x) = \sigma\bigl(T^{(j)}_1\cdots T^{(j)}_{r_j}\varphi_{j}(x) + c_j\bigr)
\end{equation}
for $j = 1,\dots,k-1$.

Let us fix $j$ and drop the superscripts to avoid notational clutter. Between each adjacent pair of Toeplitz matrices $T_i$ and $T_{i+1}$, we may insert an identity map  $I : \mathbb{R}^{n+m+1} \to \mathbb{R}^{n+m+1}$ and apply Lemma~\ref{lem:approx} to approximate $I$ by $\rho_{h_i}$ for some $h_i$ depending on  $T_i$ and $T_{i+1}$ to be chosen later. Since
\begin{align}
T_i \rho_{h_i} T_{i+1} x &= \frac{1}{h_i\sigma'(a)} T_i \sigma(h_i T_{i+1} x + a\mathbbm{1}_n) - \frac{\sigma(a)}{h_i\sigma'(a)} T_i\mathbbm{1}_n \nonumber \\
&\eqqcolon T_i' \sigma ( T_{i+1}' x  + b_{i+1}) + b_i \label{eq:remain}
\end{align}
each of these terms has the form we need. Observe that the matrices $T_i'  \coloneqq (1/h_i\sigma'(a)) \cdot T_i$ and $T_{i+1}'' \coloneqq h_i T_{i+1}$ remain Toeplitz matrices as the Toeplitz structure is invariant under scaling. We will replace each identity map between adjacent Toeplitz matrices in \eqref{eq:j} for each $i =1,\dots,r_j-1$; and then do this for each $j = 1,\dots,k-1$. By \eqref{eq:remain}, the resulting map is a $\sigma$-activated neural network with all weight matrices Toeplitz.  We will denote this neural network by $\nu$.

It remains to choose the $h_i$, or more accurately  the $h_{ij}$ since we have earlier dropped the index $j$ to simplify notation, in a way that
\[
\sup_{x\in \Omega} \;  \lVert f(x) - \nu(x) \rVert \le \varepsilon.
\]
Given that $\sup_{x\in \Omega} \lVert f(x)-\varphi(x)\rVert < \varepsilon/2$, it suffices to show
\begin{equation}\label{eq:final}
    \sup_{x\in \Omega} \;  \lVert \varphi(x) - \nu(x) \rVert \le \frac{\varepsilon}{2}.
\end{equation}
There is no loss of generality but a great gain in notational simplicity in assuming that $r_j = 2$ for $j=1,\dots,k$ and $k=2$, i.e.,
\begin{align*}
\varphi(x) &= T^{(2)}_1 T^{(2)}_2 \sigma(T^{(1)}_1 T^{(1)}_2 x + c_1) + c_2,\\
\nu(x) &= T^{(2)}_1 \rho_{h_2} T^{(2)}_2 \sigma(T^{(1)}_1 \rho_{h_1} T^{(1)}_2 x + c_1) + c_2.
\end{align*}
The reasoning is identical for the general case by repeating the argument for the $k=2=r_1=r_2$ case. Now set
\[
\psi(x) \coloneqq T^{(2)}_1 T^{(2)}_2 \sigma(T^{(1)}_1 \rho_{h_1} T^{(1)}_2 x + c_1) + c_2.
\]
We will first show that there exists $h_1 \neq 0$, such that
\begin{equation}\label{eq:h1}
\sup_{x\in \Omega} \; \lVert \varphi(x) - \psi(x)\rVert \le \frac{\varepsilon}{4}.
\end{equation}
Then we will prove that for the given $h_1$, there exists $h_2 \neq 0$ such that
\[
\sup_{x\in \Omega} \; \lVert \nu(x) - \psi(x)\rVert \le \frac{\varepsilon}{4}.
\]
By our assumption, $\sigma$ is uniformly continuous on $\mathbb{R}^n$. So there exists $\eta > 0$ such that
\[
|\sigma(a) - \sigma(b)| \le \frac{\varepsilon}{4\sqrt{n}\lVert T^{(1)}_1 \rVert\lVert T^{(1)}_2 \rVert}
\]
for any $a,b \in \mathbb{R}$ with $|a-b| \le \eta$. If we could choose $h_1 \ne 0$ so that
\begin{equation}\label{eq:int}
    \sup_{x\in \Omega} \; \lVert T^{(1)}_1 T^{(1)}_2 x  - T^{(1)}_1 \rho_{h_1} T^{(1)}_2 x \rVert \le \eta,
\end{equation}
then \eqref{eq:h1} would follow. Note that the $\sqrt{n}$ factor is necessary as $\sigma$ is applied coordinatewise to an $n$-dimensional vector.

Since $\Omega$ is compact, so is $\Omega_1 \coloneqq \{T^{(1)}_2 x : x\in \Omega\}$. Applying  Lemma~\ref{lem:approx} to $\Omega_1$ with $\eta / \lVert T^{(1)}_1\rVert$, we obtain $h_1 \neq 0$ with
\[
\sup_{y \in \Omega_1} \lVert \rho_{h_1}(y) - y \rVert \le \frac{\eta}{\lVert T^{(1)}_1\rVert}
\]
and thus
\[
    \sup_{x\in \Omega} \; \lVert T^{(1)}_1 T^{(1)}_2 x  - T^{(1)}_1 \rho_{h_1} T^{(1)}_2 x \rVert 
    \le \lVert T^{(1)}_1 \rVert \sup_{y \in \Omega_1}  \lVert \rho_{h_1}(y) - y \rVert \le \eta.
\]
Next set $\Omega_2 \coloneqq \{T^{(2)}_2 \sigma(T^{(1)}_1 \rho_{h_1} T^{(2)}_2 x + c_1) : x \in \Omega\}$, which is again compact. Applying Lemma~\ref{lem:approx} to $\Omega_2$ with $\varepsilon/4\lVert T^{(2)}_1 \rVert$, we obtain $h_2 \neq 0$ with
\[
\sup_{y \in \Omega_2} \lVert \rho_{h_2}(y) - y \rVert \le \frac{\varepsilon}{4\lVert T^{(2)}_1 \rVert}.
\]
Hence
\[
    \sup_{x\in \Omega} \; \lVert \nu(x) - \psi(x) \rVert 
   \le \lVert T^{(1)}_1 \rVert \sup_{y \in \Omega_2} \lVert \rho_{h_2}(y) - y\rVert \le \frac{\varepsilon}{4},
\]
which together with \eqref{eq:h1} gives us \eqref{eq:final} as required.

To summarize the argument, if
\[
\varphi(x) = T^{(2)}_1 T^{(2)}_2 \sigma(T^{(1)}_1 T^{(1)}_2 x + c_1) + c_2
\]
approximates $f$ to arbitrary accuracy, then we may choose $h_1$ and $h_2$ so that
\[
\nu(x) = T^{(2)}_1 \rho_{h_2} T^{(2)}_2 \sigma(T^{(1)}_1 \rho_{h_1} T^{(1)}_2 x + c_1) + c_2
\]
approximates $f$ to arbitrary accuracy and $\nu$ has all weight matrices Toeplitz. For general $k$ and $r_1,\dots,r_k$, we may similarly determine a finite sequence of $h_1,h_2,h_3,\dots$  successively and insert a copy of $\rho_{h_i}$ between each pair of Toeplitz matrices while maintaining the approximation error within $\varepsilon$. As a reminder, the inserted copy of $\rho_{h_i}$ results in a $\sigma$-activation with a bias as in \eqref{eq:rho}.

Furthermore, in the above proof, the only property of Toeplitz matrix we have used is that the Toeplitz structure is preserved under multiplication by any scalar. This scaling invariance also hold true for Hankel matrices and triangular matrices. Consequently the same arguments apply verbatim if we had used a Hankel decomposition \citep[Equation 2]{ye2016every}
\[
B_j = H^{(j)}_1H^{(j)}_2\cdots H^{(j)}_{r_j}
\]
in place of the Toeplitz decomposition in \eqref{eq:toepdecomp}. Indeed our proof extends to any decomposition of the weight matrices into a product of structured matrices whose structures are preserved under scaling.

Now there is a slight complication for the case of triangular matrices --- not every matrix will have a decomposition of the form
\begin{equation}\label{eq:LU}
B_j =  L^{(j)} U^{(j)}
\end{equation}
where $L_j$ is lower triangular and $U_j$ is upper triangular. Note that the standard LU decomposition of a matrix requires an additional permutation matrix multiplied either to the left or right \citep{GVL}. Nevertheless we could use the fact any square matrix all of whose principal minors are invertible has a decomposition of the form \eqref{eq:LU}, and since such matrices are dense in $\mathbb{R}^{n \times n}$, any matrix has an LU \emph{approximation} to arbitrary accuracy. 

For the rectangular weight matrices in the first and last layers, we note that they can be treated much in the same way as we did in the Toeplitz case. If $B$ is an $m\times n$ matrix and $m > n$, then write
\[
B = [B, \mathbbm{0}_{m\times (m-n)}] \begin{bmatrix} I_n \\ \mathbbm{0}_{(m-n) \times n} \end{bmatrix}.
\]
Since $[B, \mathbbm{0}_{m\times (m-n)}]$ is an $m \times m$ square matrix, it has an approximation  $[B, \mathbbm{0}_{m\times (m-n)}] \approx LU $ to arbitrary accuracy and therefore $B \approx LU'$  to arbitrary accuracy with $U' = U \begin{bsmallmatrix} I_n\\ \mathbbm{0}_{n \times (m-n)} \end{bsmallmatrix}$. The argument for $m > n$ is similar. In short, LU-decomposable matrices are also dense in $\mathbb{R}^{m\times n}$.

There is also an alternative approach by way of a little-known result of \citet{NDS}: Any matrix in $\mathbb{R}^{n \times n}$ can always be decomposed into a product of \emph{three} triangular matrices
\[
B_j =  L^{(j)}_1 U^{(j)} L^{(j)}_2.
\]
Note that this result may also be applied to the transpose of a matrix. So the conclusion is that any square matrix has an LUL decomposition and a ULU decomposition. The required result then follows from applying ULU decompositions to weight matrices in the odd layers and LUL decompositions to weight matrices in the even layers, adjusting for rectangular weight matrices with the argument in the previous paragraph. For example, for a neural network of the form
\[
B_2\sigma(B_1x + c),
\]
we decompose it into 
\[
L^{(2)}_1 U^{(2)} L^{(2)}_2 \sigma (U^{(1)}_1 L^{(1)} U^{(1)}_2x + c)
\]
and insert an appropriate activation between every successive factor as in the Toeplitz case to obtain an arbitrary accuracy approximation.
\end{proof}

Note that the neural network $\nu$ constructed in the proof of Theorem~\ref{thm:depth} has fixed width $m+n+1$ as in Theorem~\ref{thm:uat2} but a departure from Theorem~\ref{thm:uat2} is that $\sigma$ has to be uniformly continuous and not just continuous. Nevertheless almost all common activations like ReLU, sigmoid, hyperbolic tangent, leaky ReLU, etc, meet this requirement.

An implication of the proof of Theorem~\ref{thm:depth} is that \emph{fixed width} convolutional neural networks has the universal approximation property. While \citet{zhou2020universality} has also obtained a universal approximation theorem for convolutional neural networks, it requires arbitrary width. Our version below  requires a width of at most $m + n +1$ and, as will be evident from the proof, holds regardless of how the convolutional layers and fully connected layers in the network are ordered.

\begin{corollary}[Universal approximation theory for convolutional neural network]\label{cor:cnn}
Let $\Omega \subseteq \mathbb{R}^n$ be compact, $\sigma: \mathbb{R} \rightarrow \mathbb{R}$ be any uniformly continuous nonpolynomial function which is continuously differentiable at at least one point, with nonzero derivative at that point.  Then for any function $f \in C(\mathbb{R}^n, \mathbb{R}^m)$ and any $\varepsilon > 0$, there exists a deep convolutional neural network $\nu : \mathbb{R}^n \to \mathbb{R}^m$ with width $m+n+1$ such that
\[
\sup_{x\in \Omega} \; \lVert f(x)-\nu(x)\rVert < \varepsilon.
\]
\end{corollary}
\begin{proof}
Recall that a convolutional neural network is one that consists of several convolutional layers at the beginning and fully-connected layers consequently. Observe that in the proof of Theorem~\ref{thm:depth}, there is no need to make every layer Toeplitz --- we could replace any layer with a few Toeplitz layers or choose to keep it as is with general weight matrices while preserving the $\varepsilon$-approximation. So there is a $k$-layer neural network $g$ with first $k'$ layers Toeplitz and remaining $k-k'$ layers general such that $\sup_{x\in \Omega} \; \lVert f(x)-g(x)\rVert < \varepsilon$. Now observe that for any Toeplitz matrix
\[
T = 
\begin{bmatrix}
a_{0} & a_{-1} &  & a_{1-t}\\
a_{1} & a_{0} & \ddots & \\
& \ddots & \ddots & a_{-1}\\
a_{s-1} &  & a_{1} & a_{0}
\end{bmatrix}  \in \mathbb{R}^{s\times t},
\]
we may define a kernel $\kappa = [a_{s-1}, \dots , a_{0}, \dots, a_{1-t}]$. A layer with $T$ as weight matrix is then equivalent to a convolutional layer with kernel $\kappa$ and stride $1$. By doing this to every Toeplitz layer in $g$, we transform it into a convolutional neural network $\nu$ with $k'$ convolutional layers and $k - k'$ fully connected layer.
\end{proof}

\section{Training cost analysis}\label{sec:cost}

Here we perform a basic estimate of how much savings one may expect from imposing an LU or Toeplitz/Hankel structure on a neural network. The reduction in weight parameters is the most obvious advantage: an $m \times n$ upper triangular matrix requires $(n+1)n / 2$ parameters if $m \ge n$ and $(2n-m+1) m / 2$ if $m < n$; an $m \times n$ lower triangular matrix requires $(2m-n+1)n / 2$ parameters if $m \ge n$ and $(m+1)m/2$ if $m < n$; an $m \times n$ Toeplitz or Hankel matrix requires just $m + n -1$ parameters. However there is also a slightly less obvious advantage that we will discuss next.

The standard basic procedure in training a neural network involves a loss function $\ell$ on the output of network. Common examples include cross entropy loss, mean squared error loss, mean absolute error loss, negative log likelihood loss, etc. We calculate the gradient of $\ell$ under each weight parameter, and then update each parameter with the corresponding gradient scaled by a learning rate. The training process comprises two parts, forward propagation and backward propagation. In forward propagation, the neural network is evaluated to produce the output from the input. The computational cost is dominated by the matrix-vector multiplication in each layer:
\[
y_i = A_iz_i + b_i, \qquad z_{i+1} = \sigma(y_i)
\]
In backward propagation, we calculate the gradient of each parameter wwith chain rule. In the $i$th layer, the gradient is calculated from
\[
    \nabla_{z_i} \ell = A_i^\tp \nabla_{y_i} \ell, \qquad
    \nabla_{A_i} \ell = (\nabla_{y_i} \ell) \otimes z_i,
\]
where $\otimes$ denotes outer product.  Again, the computational cost is dominated by matrix-vector multiplication in each layer.

Given that training cost ultimately boils down to matrix-vector multiplications, we expect massive savings by exploiting such algorithms for structured matrices, particularly in the Toeplitz or Hankel cases, as these matrix-vector products can be computed in $O(n\log n)$ complexity, compared to the usual $O(n^{2})$ for general matrices. But even triangular matrices would immediately halve the cost of training.

\section{Experiments}\label{sec:exp}

We have conducted extensive experiments to demonstrate that neural networks with structured weight matrices such as those discussed in this article are almost as accurate as general ones. For a fair comparison, in each experiment we fixed the width and depth of the neural networks, changing only the type of weight matrices used, whether general (i.e., no structure), triangular, Toeplitz, or Hankel. In particular, all weight matrices have same dimensions, differing only in their structures or lack therefore. We have also taken care to avoid over-fitting in all our experiments, to ensure that we are not comparing one overfitted neural network with another. One telling sign of over-fitting is poor test accuracy, but all our experiments, test accuracy is reasonably high.

We performed our experiments with three common data sets: MNIST comprises a training set of 60,000 and a test set of 10,000 handwritten digits. CIFAR-10 comprises 60,000 $32\times32$ color images in 10 classes, with 6,000 images per class, divided into a training set of 50,000 and a test set of 10,000. WikiText-2 is a collection of over 100 million tokens extracted from verified `Good' and `Featured' articles on Wikipedia. 

We used our neural networks in three different contexts: as \emph{multilayer perceptrons}, i.e., the classic feed forward neural network with fully connected layers; as \emph{convolutional neural networks} that have convolutional, pooling, and fully connected layers  \citet{lecun1998gradient}; and as \emph{transformers}, a widely-used architecture based solely on attention mechanisms \citep{vaswani2017attention}.

\subsection{MNIST and multilayer perceptron:} For an image classification task with MNIST, we compare a three-layer multilayer perceptron with three general weight matrices against one where the three weight matrices are upper, lower, and upper triangular respectively; and another where all three weight matrices are Toeplitz.  We use a cross entropy loss, set learning rate to $0.01$, batch size to $20$, and trained for $50$ epochs. The mean, minimum, and maximum accuracy of each epoch over five runs are reported in Figure~\ref{mnist}. Our results show that the LU neural network has similar performance as the general neural network on both training accuracy and test accuracy. While the Toeplitz neural network sees poorer performance, its test accuracy, at greater than 95\%, is within acceptable standards.
\begin{figure}[htb]
    \centering
    \includegraphics[width=0.6\textwidth,trim={3ex 1ex 9ex 8ex},clip]{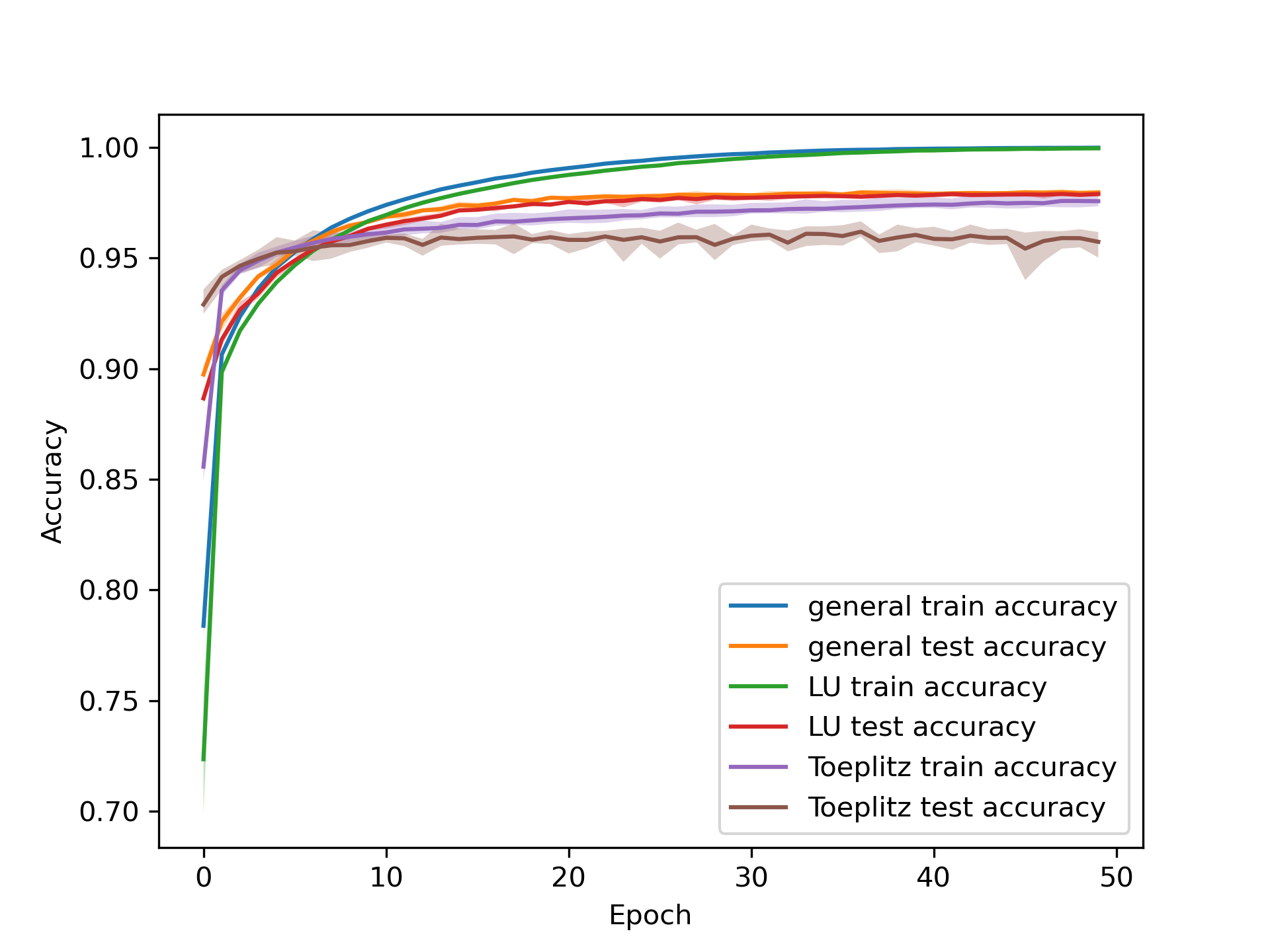} 
    \caption{Accuracy on MNIST}
    \label{mnist}
\end{figure}

\subsection{CIFAR-10 and convolutional neural networks:} For another image classification task with CIFAR-10, we compared a three-fully-connected-layer AlexNet \citep{krizhevsky2017imagenet} with three general weight matrices to one with three triangular weight matrices and another with three Toeplitz weight matrices. We set learning rate at $0.01$, batch size at $32$, and trained for $100$ epochs. The results are in Figure~\ref{cifar}. In this case, we see no significant difference in the performance --- LU AlexNet and Toeplitz AlexNet do just as well as the usual AlexNet.

\begin{figure}[htb]
    \centering
    \includegraphics[width=0.6\textwidth,trim={4ex 1ex 9ex 8ex},clip]{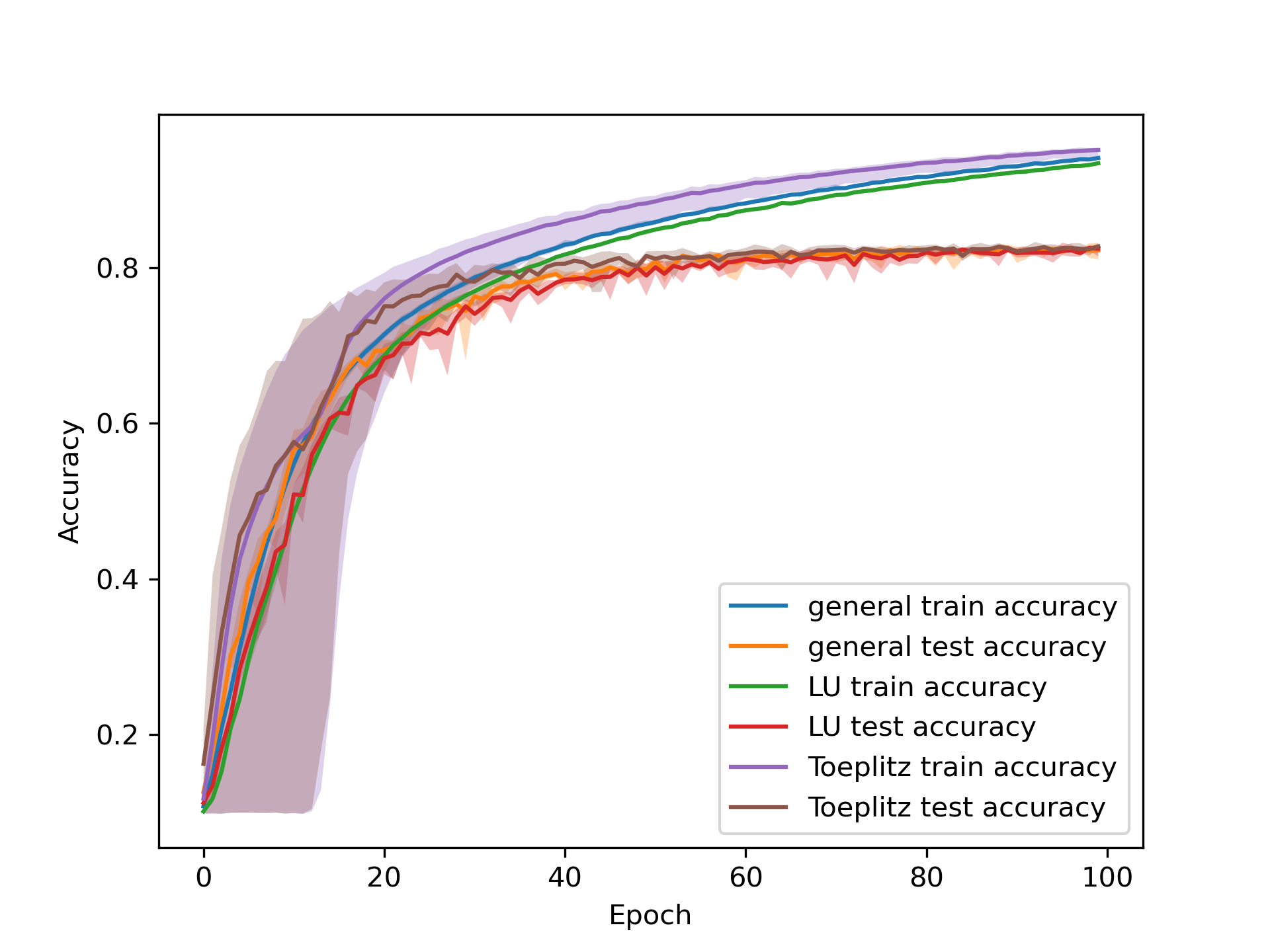} 
    \caption{Accuracy on CIFAR-10}
    \label{cifar}
\end{figure}

\subsection{WikiText and transformer:} We use a transformer with a two-head attention structure for a language modeling task with WikiText-2. As before, we compare three versions of the transformer where the fully connected layers are either general, LU, or Toeplitz neural networks.  We use a batch size of $20$, a learning rate of $5$, decaying by $0.2$ for every $10$ epochs. The mean, minimum, and maximum perplexity of each epoch over five runs are reported in Figure~\ref{wiki}. Recall that perplexity is the exponential of cross entropy loss, and thus a lower value represents a better result. Here the LU transformer performs as well as the standard transformer; the Toeplitz transfomer, while slightly less accurate, is nevertheless within acceptable standards.

\begin{figure}[htb]
    \centering
    \includegraphics[width=0.6\textwidth,trim={5ex 1ex 9ex 8ex},clip]{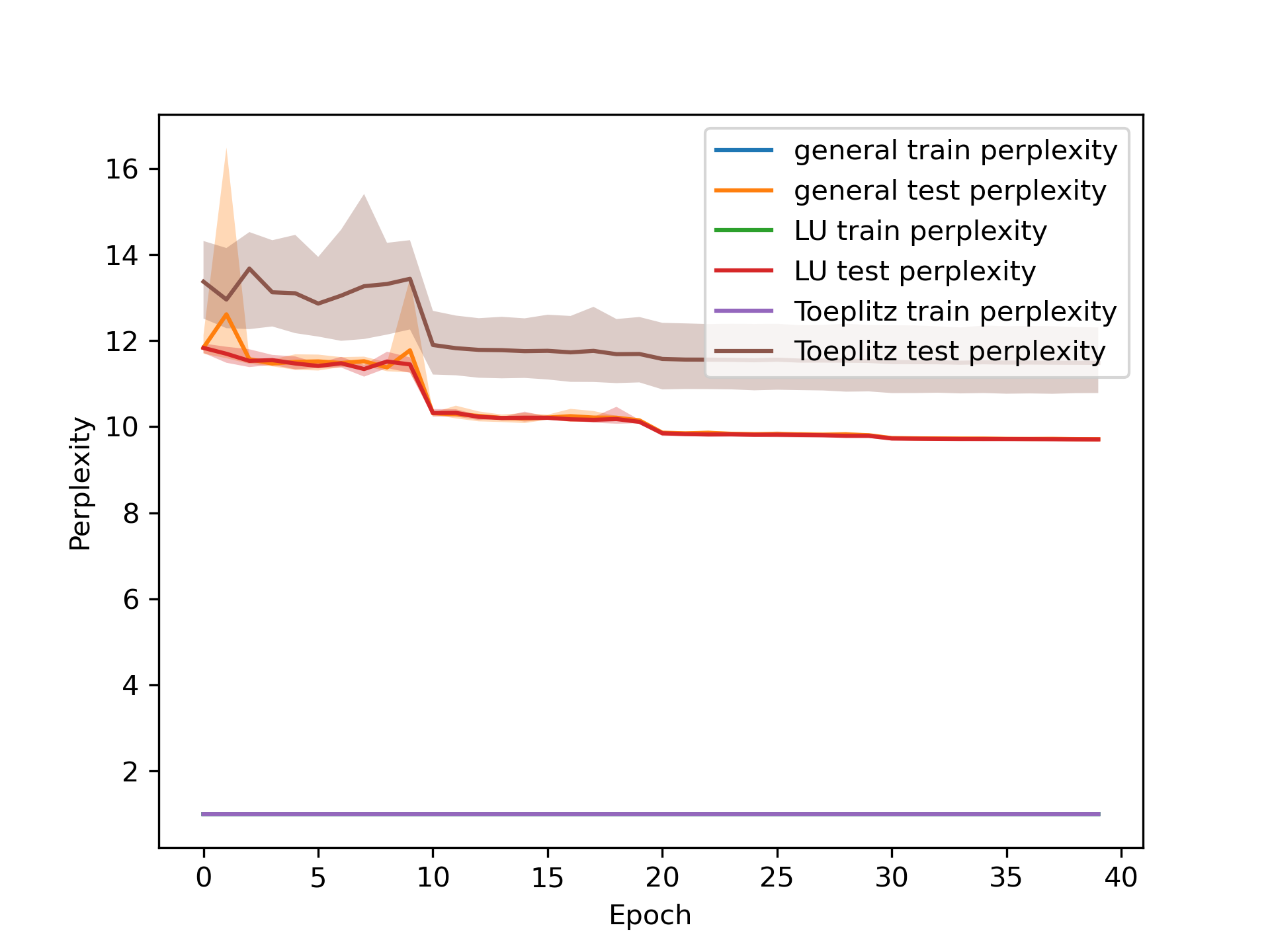} 
    \caption{Perplexity on Wiki-Text 2}
    \label{wiki}
\end{figure}

\section{Conclusion} Our results here may be viewed as a first step towards extending the standard matrix decompositions --- widely regarded as one of the top ten algorithms of the 20th century \citep{Stewart} --- from linear maps to continuous maps. Viewed in this light, there are many open questions: Is there a reasonable way to extend QR decomposition or singular value decomposition in a manner similar to what we did for LU and Toeplitz decompositions? Could one compute such decompositions in a principled way like their linear counterpart as opposed to fitting them with data? Can one design neuromorphic chips with lower energy cost or with lower gate complexity by exploiting such decompositions?

\subsection*{Acknowledgments} This work is partially supported by the DARPA grant HR00112190040 and the NSF grants DMS 1854831 and ECCS 2216912.

\bibliographystyle{abbrvnat}
\bibliography{references}
\end{document}